\documentclass[12pt]{article}
\usepackage{amsmath,amsthm,amsfonts,latexsym,amssymb,amscd,color}
\usepackage{tikz}
\usepackage{algorithm}
\usepackage{authblk}
\usepackage{url}
\usepackage{amsmath}
\usepackage{faktor}
\usepackage{enumitem}
\usepackage{listings}
\usepackage{mdframed}
\usetikzlibrary{chains}
\usepackage{tikz}
\usetikzlibrary{automata, positioning, arrows}
\usepackage{amsthm}
\usepackage{hyperref}
\usepackage{xstring}
\theoremstyle{definition}
\newtheorem{example}{Example}

\newmdenv[linewidth=2pt,roundcorner=10pt,backgroundcolor=white,font=\bfseries]{questionbox}

\usepackage[linguistics]{forest}
\usepackage{float}
\usepackage{array}
\usepackage{caption}
\usepackage{algorithmicx}
\usepackage{algpseudocode}
\newlength\myindent
\let\oldReturn\Return
\renewcommand{\Return}{\State\oldReturn}
\newtheorem{thm}{Theorem}[section]
\newtheorem{cor}[thm]{Corollary}
\newtheorem{lm}[thm]{Lemma}

\newtheorem*{clm*}{Claim}
\newtheorem*{remark}{Remark}
\newcommand{\cproof}{\noindent{\it Proof of Claim.}\ } 
\newcommand{\cqed}{\hfill\rule{1.3mm}{3mm}}

\theoremstyle{definition}
\newtheorem{df}[thm]{Definition}

\numberwithin{equation}{section}

\DeclareMathOperator*{\argmax}{arg\,max}

\definecolor{mypink2}{RGB}{219, 48, 122}

\newcounter{constant}

\title{Optimal L-Systems for Stochastic L-system Inference Problems}

\author[1]{Ali Lotfi}
\author[2]{Ian McQuillan}
\affil[1]{University of Saskatchewan, Department of Computer Science}
\affil[2]{University of Saskatchewan, Department of Computer Science}
\begin{document}
	\maketitle

\begin{abstract}
This paper presents two novel theorems that address two open problems
in stochastic Lindenmayer-system (L-system) inference, specifically focusing
on the construction of an optimal stochastic L-system capable of generating
a given sequence of strings. The first theorem delineates a method for crafting a stochastic L-system that has the maximum probability of a derivation producing a given
sequence of words through a single derivation (noting that multiple derivations may generate the same sequence). Furthermore, the second
theorem determines the stochastic L-systems with the highest probability of
producing a given sequence of words with multiple possible derivations. From
these, we introduce an algorithm to infer an optimal stochastic L-system from
a given sequence. This algorithm incorporates advanced optimization techniques, such as interior point methods, to ensure the creation of a stochastic L-system that maximizes the probability of generating the given sequence (allowing for multiple derivations). This
allows for the use of stochastic L-systems as a model for machine learning using only positive data for training. 
\end{abstract}

\section{Introduction}
Lindenmayer systems, abbreviated L-systems, were originally introduced by Aristid
Lindenmayer in 1968 \cite{lindenmayer1968mathematical}. They have been widely used to model the growth and development of various biological systems, particularly plant morphology \cite{prusinkiewicz2012algorithmic}. Letters of an L-system can be associated with drawing instructions, enabling L-system simulators such as vlab \cite{algorithmicbotany} to produce accurate graphical depictions of a plant’s development over time. 

L-systems can be deterministic where only one possible rule can be applied to each letter (possibly in certain contexts or conditions for context-sensitive L-systems). Deterministic L-systems only have one possible derivation. However, to model larger sets of events that can all occur 
(or for plants, populations of plants, or for an entire species), nondeterminstic L-systems are needed where more than one rule can be applied to any letter. 
Stochastic L-systems extend this concept by incorporating probabilistic rules \cite{prusinkiewicz1996stochastic}.

Stochastic L-system inference problems have numerous applications, with one of the most notable being the creation of stochastic L-systems that can simulate plant growth and produce structures akin to actual plants. A salient example of this is seen in \cite{nishida1980k0l}. In this particular study, the authors introduced a probabilistic 0L-system to explore the growth of scale leaves in the Japanese Cypress, Chamaecyparis obtusa. They crafted tree-like formations mirroring the diversity found in real trees. Upon observing the Japanese Cypress’s growth patterns, they painstakingly developed accurate probabilistic rewriting rules. The end result was a set of rules that adeptly mimicked real tree structures. The study also highlighted certain statistical attributes of the branching designs, such as the number of symbols and the occurrences of irregular and double branchings. The simulated tree branching patterns are closely aligned with those of natural trees, encompassing their varied forms.

Another important application of stochastic L-systems is that they can be used to generate synthetic images that we can be used to train artificial neural networks for the purposes of computer vision in plants. Indeed, supervised machine learning approaches require a large amount of annotated data, and with L-systems, any annotation labels built into the simulation are implicitly known. Therefore, simulated images can often be used on their own, or in combination with real annotated images to train artificial neural networks, thereby reducing or eliminating the time consuming process of labeling data. This has been successfully used in \textit{Arabidopsis thaliana} \cite{ubbens2018use}, and the general process of manually creating an L-system to fit a plant is described in \cite{cieslak2022system}.

Normally, constructing an L-system of a developing plant requires an expert in constructing L-systems, and the process is difficult and time consuming \cite{cieslak2022system}. If this could be done instead algorithmically from a sequence of images, it would dramatically open up the practicality of creating L-systems for a given biological process. As an intermediate step, one could infer from a sequence of strings generated from an unknown L-system, where those string could draw the simulated images with the simulator. However, even constructing an optimal stochastic L-system from a given sequence of strings is a challenging problem with significant practical implications for modeling and simulating complex
natural phenomena. Techniques for inferring different types of L-systems were surveyed in \cite{ben2009survey}. The majority of recent inference papers have focused on inferring deterministic L-systems, where there is only one possible derivation \cite{bernard2021techniques}. There has been one paper on inferring stochastic context-free L-systems \cite{bernard2023stochastic}. This paper focuses on a heuristic method for inferring L-systems without any mathematical guarantees of correctness.

This paper is a step towards an automated solution, as it presents two novel
theorems and an algorithm to address this challenge. The first theorem focuses on constructing a stochastic L-system that has a derivation with maximum probability of generating a given sequence of strings. The second theorem identifies the stochastic L-systems with the highest probability of generating the given sequence of strings. In this latter problem, the probability of generating a sequence of words is the sum of probabilities of generating the sequence across all possible derivations. With these theorems, an algorithm is proposed that incorporates nonlinear programming solvers using optimization techniques described in \cite{boyd2004convex}. This opens up the possibility of using stochastic L-systems as a machine learning technique. Here, the stochastic L-System is being used differently than the indirect approach using artificial neural networks as described above. Using L-Systems could be suitable for modeling tasks with parallel rewriting. Indeed the rules and probabilities are automatically inferred from only positive data used for training (a correct sequence of strings). The
parallel rewriting allows for considerable structure in strings so that only using positive data is worthwhile.

This paper is organized as follows: Section \ref{def-l} provides the necessary background on L-systems and stochastic L-systems, and introduces new notations. Section \ref{describe-problem} formally describes the two problems we will solving and presents the mathematical tools used to tackle the main problems. Sections \ref{sec:main-thm1} and \ref{sec:main-thm2} answer Questions 1 and 2 respectively and presents and proves novel theorems. Finally, Section \ref{sec:alg} introduces the algorithm for constructing an optimal stochastic L-system described in Theorem~\ref{thm:main1}.

\section{Preliminaries and Notation}\label{def-l}

In this section, we will define 0L-systems, D0L-systems and Stochastic 0L-systems, also known as S0L systems. Subsequently, we introduce new notations that will simplify the mathematical approach. L-Systems are defined using concepts from formal language theory, and we will give some definitions from that area. An alphabet $V$ is a finite set of symbols. Furthermore, $V^*$ represents the set of all strings using letters from $V$. A language over $V$ is a subset of $V^*$. Given any string $w$, we use the notation $|w|$ to refer to its length, and $w[i]$ to refer to the ith ($1\leq i \leq |w|$) character of $w$. We also refer to the cardinality of a finite set $X$ as simply $|X|$, which represents the number of elements in $X$. Given a function $f$, we denote its domain and codomain by $dom(f)$ and its $cod(f)$ respectively.

\subsection{0L-systems and D0L-systems}

In this part, we will define L-systems and notations similar to \cite{mcquillan2018algorithms} and \cite{prusinkiewicz2012algorithmic}. An L-
system that is context-free, or a 0L-system, applies productions to symbols without regard to their context within a string.  It is denoted by $G = (V, \omega, P)$, where $V$ is an alphabet, $\omega \in V^*$ is the axiom, and $P \subseteq V \times V^*$ is a finite set of productions. A production $(a, x) \in P$ is denoted by $a \rightarrow x$. The symbol $a$ is called the production predecessor, and the string $x$ is its successor. We assume that for each predecessor $a \in V$, there is at least one production $a \rightarrow x$ in $P$. However, for mathematical convenience, we sometimes consider variants of L-systems where this property is not true, and in this case, we call the L-system {\it partial}. If for each $a \in V$ there is exactly one such production, then the 0L-system $G$ is said to be deterministic, or a D0L-system. Given a string $a = a_1 \cdots a_n \in V^*$, we use the notation $a \Rightarrow x$ and say that $a$ directly derives $x$ if $x = x_1 \cdots x_n$, where $a_i \rightarrow x_i \in P$ for all $1 \leq i \leq n$. A derivation $d$ in $G$ consists of the following:

\begin{enumerate}
	\item A trace, which is a sequence $(w_0, \ldots, w_m)$ such that $w_{i-1}$ directly derives $w_{i}$ for $1 \leq i \leq m$ or in other words $w_0 \Rightarrow w_1 \Rightarrow \cdots \Rightarrow w_m$,
	\item a function $\sigma_{d}$ from $\{(j,l) \mid 0 \leq j < m , 1\leq l \leq |w_j| \}$ into $P$ such that if $w_j = a_1\cdots a_{|w_j|}$, $0\leq j<m$ and $1\leq l \leq |w_j|$, then we have that $w_{j+1} = x_1\ldots x_{|w_j|}$ where $\sigma_{d}(j,l) = a_{l} \rightarrow x_l$.
\end{enumerate}

Note that $\sigma_d$ describes the productions that are applied to each letter by the derivation. The set of productions $\{\sigma_{d}(j,l) \mid 0 \leq j < m , 1\leq l \leq |w_j| \}$ is the set of
productions used in a derivation $d$ and is denoted by $P(d)$. Furthermore, we will refer to the number of times a production $a\rightarrow x$ occurs in a derivation, as $\#_{d,a\rightarrow x}$.
We say that $G$ can generate a sequence of words $\theta$ if there exists a derivation $d$ in $G$ with $\theta$ as trace. In this case, we say $G$ can generate $\theta$ with $d$. Given a sequence $\theta=(w_0,\ldots,w_m)$
and 0L system G, if $G$ can generate $\theta$ as a trace, we say $\theta$ is compatible with $G$, and incompatible otherwise. Furthermore, we denote the set of derivations of $G$ that can generate $\theta$ as $Der_{G}(\theta)$. We will also introduce new notations to describe the relationship between the construction and partitioning of strings. For a string $x$ with $|x| = m$, we define $Enum(x)$ to be an enumeration of the characters appearing in $x$ as
follows:

\[Enum(x) = \{ (1,x[1]),\ldots,(m,x[m])\}.\]
Consider two strings, $x$ and $y$ such that $|x| = m$ and $|y| = n$, define $Par(x, y)$ as the set of all possible partitions of $y$ into $m$ substrings which some could be the empty string ($\varepsilon$). Specifically, we have:
\[Par(x,y) = \{ \{(1,y_1),\ldots,(m,y_{m})\} \mid y_1\cdots y_{m}=y\}.\]
Furthermore, we define $Par.Func(x, y)$ as:
{\small
	\[Par.Func(x,y) = \{ f:Enum(x)\mapsto Cod(f) \mid Cod(f) \in Par(x,y),f\left( \left( i,x[i] \right) \right) = \left(i,y_{i}\right)\}.\]

 Intuitively, $Par.Func(x,y)$ is the set of all functions that map the pair consisting
of a position in $x$ and its corresponding letter (such as $(i, x[i])$) to the i-th subword
of $y$ called $y_i$ , where $y_1\cdots y_{i-1}y_{i}y_{i+1}\cdots y_{n} = y$.
Now, let's move on to defining $Cand.Prod$, which represents the set of all productions that can be used to obtain a sequence. Specifically, we define $Cand.Prod(x, y)$ as:
{\small
 \[Cand.Prod(x,y) = \{ x[i] \rightarrow y_i \mid f \in Par.Func(x,y), f((i,x[i])) = (i,y_i) \text{ for } 1\leq i\leq m\}.\]
}
\begin{remark}
	Given an 0L-system $G$, if string $x$ is mapped to $y$ in one step through a derivation $d$, then $d$ can be transformed into an element of $Par.Func(x,y)$.  
\end{remark} 
\begin{df}\label{df:free-0l}
	Given a sequence $\theta=(w_0, \ldots, w_m)$ over $V$, we define $\mathcal{L}_{\theta}=(V, \omega, P)$ to be the free partial 0L system over $\theta$ if:
	\begin{enumerate}
		\item $\omega=w_0$
		\item $P= \bigcup_{0\leq i<m} Cand.Prod(w_i,w_{i+1})$\label{df:free-0l-itm1}
	\end{enumerate} 
\end{df}

Notice that in Item \ref{df:free-0l-itm1} of Definition \ref{df:free-0l}, for a sequence $\theta$, letter $A$ has no productions in the free 0L system if and only if letter $A$ only occurs in $w_m$ or no words. That is no, data is provided regarding what $A$ could be rewritten to as string $w_{m+1}$ was provided as input.
\subsection{Stochastic 0L-systems}

In this section, we define stochastic L-systems. The definitions we will provide have been thoroughly discussed in \cite{bernard2023stochastic,prusinkiewicz2012algorithmic}. A context-free stochastic L-system (S0L-system) is a quadruplet $G = (V, \omega, P, p)$. The alphabet $V$, axiom $\omega$, and set of productions $P$ are defined as in a standard 0L-system. Additionally, $G$ is equipped with a function $p : P \rightarrow (0, 1]$ that assigns to each production a probability of application. For any letter $a \in V$, it is assumed that the sum of probabilities of all productions with $a$ as predecessor is equal to 1. In other words, for each $a \in V$, $\sum_{ {a\rightarrow x \in P } }p(a\rightarrow x) = 1$. The definition of a derivation for a stochastic L-system is the same as that of a 0L-system,
but now there is a probability assigned to derivations. If the derivation $d$ has trace $\theta = (w_0, \ldots, w_m)$ and a function $\sigma_d$, then:

\[p(w_j \Rightarrow w_{j+1};d) = \prod_{l=1}^{|w_j|} p(\sigma_{d}(j,l)),\]
Furthermore,
\[p(d) = \prod_{j=0}^{m-1}\prod_{l=1}^{|w_j|} p(\sigma_{d}(j,l)).\]
Given a sequence \( \theta = (w_1, \ldots, w_m) \) over \( V \) and an S0L system \( G = (V, \omega, P, p) \),

\[ p(\mathcal{\theta}) = \sum_{d \in Der_G(\mathcal{\theta})} p(d). \]
Indeed, if we are given only \( \theta \), then there could be multiple derivations with \( \theta \) as trace, and the probability of $G$ generating \( \theta \) is the sum of the probabilities of each derivation of \( G \) with \( \theta \) as trace.

To provide a clearer understanding of S0L-systems, consider the following examples of S0L-systems generating a sequence. In the first example, the S0L generates the sequence through a single derivation, while in the second, it generates the sequence through two derivations. In both examples, characters of a certain color are mapped to substrings of the same color, indicating which production has been used.

\begin{example}
	Let $G_1 = (V, \omega, P, p)$, where $V = \{A, B, C\}$, $\omega = AAA$, and $P = \{A \rightarrow ABA, A \rightarrow B, A \rightarrow AC, B \rightarrow B, C \rightarrow C\}$, with $p(A \rightarrow ABA) = \frac{1}{3}$, $p(A \rightarrow B) = \frac{1}{3}$, and $p(A \rightarrow AC) = \frac{1}{3}$. Then $G_1$ generates the following two-step sequence with probability $\frac{1}{27}$:
	$$ {\color{red} A} {\color{blue} A} {\color{cyan} A} \Rightarrow {\color{red} ABA} {\color{blue} B} {\color{cyan} AC} $$
\end{example}

\begin{example}
	Let $G_2 = (V, \omega, P, p)$, where $V = \{A, B, C\}$, $\omega = AA$, and $P = \{A \rightarrow AB, A \rightarrow BA, A \rightarrow A, B \rightarrow B, C \rightarrow C\}$, with $p(A \rightarrow AB) = \frac{1}{3}$, $p(A \rightarrow BA) = \frac{1}{3}$ and $p(A \rightarrow A) = \frac{1}{3}$. Then $G_2$ generates the following two-step sequence in two different ways:
	$$ AA \Rightarrow ABA $$
	$$ {\color{red} A} {\color{blue} A} \Rightarrow {\color{red} AB} {\color{blue} A}, \textbf{\quad with probability} \quad \frac{1}{9}$$
	$$ {\color{red} A} {\color{blue} A} \Rightarrow {\color{red} A} {\color{blue} BA}, \textbf{\quad with probability} \quad \frac{1}{9}$$
	with total probability $\frac{2}{9}$.
\end{example}

\section{Main Questions and Mathematical Tools}\label{describe-problem}
Now, we can formally propose the two open questions related to S0L inference that are of interest. We intend to answer the following questions:

\begin{questionbox}
	Question 1: Given alphabet $V$, and a sequence $\theta = (w_0, \ldots, w_m)$ over $V$, construct an S0L-system $G = (V, \omega, P, p)$ and a derivation $d$ in $G$ that can generate $\theta$ with $d$ and $p(d)$ is maximal among all derivations $d^{\prime}$ from any S0L-systems $G^{\prime}$ that generate $\theta$ with $d^{\prime}$? 
\end{questionbox}

\begin{questionbox}
	Question 2  : Given alphabet $V$, and $\theta = (w_0, \ldots, w_m)$ over $V$, construct a S0L-system $G$ such that $p(\theta)$ is maximal over S0L-systems that can generate $\theta$?  
\end{questionbox}
Notice that the two question do look similar. However, in Question 2 we are looking for a S0L-system which has the highest probability of generating the sequence perhaps through many different derivations. On the other hand, in Question 1, we are looking for an S0L which has a derivation $d$ with the highest probability of generating the
sequence.

Hereafter, we provide the mathematical background necessary to address both questions. For the first question, Lemma~\ref{am-gm-main} is required. A detailed solution to the relevant optimization problem is available at \href{https://math.stackexchange.com/questions/2028410/maximize-prod-k-1m-x-in-i-subject-to-c-sum-k-1m-a-ix-i}{this link}, and we have independently verified its correctness.

\begin{lm}\label{am-gm-main}
	Let \( \{x_i\}_{i=1}^m \) be a set of variables constrained by 
	\[
	\sum_{i=1}^m a_i x_i = C, \quad x_i \geq 0, \quad a_i > 0, \quad C > 0,
	\]
	where \( a_i \) are constants. Then the product \( \prod_{i=1}^m x_i^{n_i} \) is maximized when
	\[
	x_i = \frac{C n_i}{N a_i}, \quad \text{for all } i = 1, \ldots, m,
	\]
	where $N = \sum_{j=1}^m n_j$.
	The maximum value of \( \prod_{i=1}^m x_i^{n_i} \) under these constraints is
	\[
	\left( \frac{C}{N} \right)^N \prod_{i=1}^m (\frac{n_i}{a_i})^{n_i}.
	\]
\end{lm}

In order to answer Question 2, we will introduce and work with minimizing or maximizing posynomials which are of the following form:
\[p(x) = \sum_{i=1}^n c_i x_1^{a_{i,1}} x_2^{a_{i,2}} \cdots x_m^{a_{i,m}}.\]

To answer the second question, we will translate our problem into an optimization problem of maximizing a posynomial subject to a series of posynomial constraints. Once this translation has been done, optimal solutions could be approximated using nonlinear programming solvers, such as the interior point method(\cite{boyd2007tutorial}). In our work,
we will use the Posy-Solver implementation described in \cite{boyd2007tutorial}. In case of Question 1, having Lemma~\ref{am-gm-main} will help lead to Theorem \ref{thm:main1} and Corollary \ref{cor:weight-dist} below. For Question 1, we do not need to consider the space of continuous probabilities for different productions coming from different derivations, and can for example perform a search on the discrete space of derivations to find $d^{*}$ in Corollary \ref{cor:weight-dist}.

\section{Determining the Optimal Derivation}\label{sec:main-thm1}

We are now ready to answer the problem posed in Question 1, which involves constructing a stochastic L-system that has a derivation with the highest probability of deriving the given sequence. First, we will prove Theorem \ref{thm:main1}, which shows that the probability of generating a sequence by a derivation is bounded above by a certain value, which we will show is sharp. Furthermore, Corollary \ref{cor:weight-dist} demonstrates how the probabilities should be distributed for this upper bound to be obtained.

\begin{thm}\label{thm:main1}
	
	Let $\theta =(w_0, \ldots, w_m)$, $V$ be the characters in $\theta$, $ G = (V, \omega, P, p)$ an
	S0L system, and $d$ a derivation which generates $\theta$. Then:
	\[p(d) \leq \left(\prod_{v \in V} \frac{1}{\sum_{i=0}^{m-1}|w_i|_{v}^{\sum_{i=0}^{m-1}|w_i|_{v}}}\right) 
	\left( \prod_{a\rightarrow y \in P(d)} \#_{d,a\rightarrow y }^{\#_{d,a\rightarrow y }}\right).\]
\end{thm}

\begin{proof}
	We can assume that for any $a\rightarrow y \notin P \setminus P(d)$, $p(a\rightarrow y)=0$ since such productions either decrease $p(d)$ or do not change it, and therefore, we will assume they are all zero. Let multivariate polynomial $f(\overline{X})$ be $\prod_{a\rightarrow y \in P(d)} X_{a\rightarrow y}^{\#_{d,a\rightarrow y}}$
	, and for a fixed $a \in V$, define $f_{a}(\overline{X}_{a}) = \prod_{a\rightarrow y \in P(d)} X_{a\rightarrow y}^{\#_{d,a\rightarrow y}}$, then we have:
	
	\[f(\overline{X}) = \prod_{a \in V}f_{a}(\overline{X}_{a}).\]  
	Furthermore, for each $a \in V$, we have:
	\[\sum\limits_{\substack{a\rightarrow y \in P(d)}} X_{a\rightarrow y} = 1,\]
	and since for $a,b \in V$ where $a\neq b$, $f_a$ and $f_b$ are not sharing common variables, then the following optimization problems are independent for different elements of $V$ such as $a$:
	\begin{align*}
		&\textbf{Maximize: }  \prod\limits_{\substack{a\rightarrow y \in P(d)}} X_{a\rightarrow y}^{\#_{d,a\rightarrow y}}\\
		&\textbf{Subject to: $\left(\sum_{a\rightarrow y \in P(d)} X_{a\rightarrow y}= 1\right)$ and } \left( \forall a\rightarrow y \in P(d)~,~  X_{a\rightarrow y} \geq 0 \right)\\	
	\end{align*}
	Then, for each \( a \in V \) and its corresponding optimization problem, we can apply Lemma~\ref{am-gm-main} (by setting $N = \sum\limits_{\substack{a\rightarrow y \in P(d)}} \#_{d,a\rightarrow y}$, $a_i = 1$ for $1 \leq i \leq N$, and $C = 1$), and obtain the optimal solution which is as following:
	\[X^*_{a\rightarrow y } = \frac{\#_{d,a\rightarrow y }}{\sum\limits_{\substack{a\rightarrow y \in P(d)}} \#_{d,a\rightarrow y }}\text{ for } a\rightarrow y \in P(d).\]
	Notice that $\sum\limits_{\substack{a\rightarrow y \in P(d)}} \#_{d,a\rightarrow y }$ is the number of productions used in $d$ which have $a$ as predecessor which is the same as number of times character $a$ appears in one of the strings $w_0,\ldots,w_{m-1}$. Therefore; $\sum\limits_{\substack{a\rightarrow y \in P(d)}} \#_{d,a\rightarrow y } = \sum_{i=0}^{m-1}|w_i|_{a}$. As a result, we have the following:
	\[X^*_{a\rightarrow y } =  \frac{\#_{d,a\rightarrow y }}{\sum_{i=0}^{m-1}|w_i|_{a}} \text{ for } a \rightarrow y \in P(d).\]
	This means that for $a \in V$, if $\sum_{a\rightarrow y \in P(d)} X_{a\rightarrow y}= 1$, then by Lemma~\ref{am-gm-main} we have:
	\[f_{a}(X_a) \leq f_{a}(\overline{X}^*_{a}) = \frac{1}{\sum_{i=0}^{m-1}|w_i|_{a}^{\sum_{i=0}^{m-1}|w_i|_{a}}}\prod_{a\rightarrow y \in P(d)} \#_{d,a\rightarrow y }^{\#_{d,a\rightarrow y }}.\]
	Since $f(\overline{X}) = \prod_{a \in V}f_{a}(\overline{X}_{a})$, and $f_{a}(\overline{X}_{a})$s do not share variables for different $a\in V$, then $\{\overline{X}^*_{a} : a\in V\}$ is a solution to the following optimization problem:
	\begin{align*}
		&\textbf{Maximize: }  f(\overline{X})\\
		&\textbf{Subject to: $\left(\sum_{a\rightarrow y \in P(d)} X_{a\rightarrow y}= 1\right)$ \text{for all} $a \in V$}\\
		&\textbf{Subject to: $ X_{a\rightarrow y} \geq 0$ \text{for all} $a\rightarrow y \in P(d)$}\\	
	\end{align*}
	On the other hand, take $Y_{a\rightarrow y}= p(a\rightarrow y)$ and notice that since $G_p$ is an S0L, then $\overline{Y}$ satisfies the constraints of the last optimization problem, and we also have $p(d) = f(\overline{Y})$, therefore, we have:
	\begin{eqnarray}
		\nonumber p(d) = f(\overline{Y}) &\leq& f(\overline{X}^*)\\
		\nonumber &=& \prod_{a\in V}\left( \frac{1}{\sum_{i=0}^{m-1}|w_i|_{a}^{\sum_{i=0}^{m-1}|w_i|_{a}}}\prod_{a\rightarrow y \in P(d)} \#_{d,a\rightarrow y }^{\#_{d,a\rightarrow y }}\right)\\
		\nonumber &=& \prod_{a\in V}\left( \frac{1}{\sum_{i=0}^{m-1}|w_i|_{a}^{\sum_{i=0}^{m-1}|w_i|_{a}}}\right) \prod_{a\in V} \left(\prod_{a\rightarrow y \in P(d)} \#_{d,a\rightarrow y }^{\#_{d,a\rightarrow y }}\right)\\
		\nonumber &=& \prod_{a\in V}\left( \frac{1}{\sum_{i=0}^{m-1}|w_i|_{a}^{\sum_{i=0}^{m-1}|w_i|_{a}}}\right)  \prod_{a\rightarrow y \in P(d)} \#_{d,a\rightarrow y }^{\#_{d,a\rightarrow y }}\\
	\end{eqnarray}
	finishing the proof.
\end{proof}

\begin{cor}\label{cor:weight-dist}
	
	Let $\theta=(w_0, \ldots, w_m)$ be a sequence, and $V$ be the characters in the $\theta$, then the derivation $d^*$ and $L=(V, \omega, P^{*} , p^{*})$ has the highest probability of generating $\theta$ if:
	\begin{enumerate}
		\item $d^* =\argmax_{d \in Der_{\mathcal{L}_{\theta}}(\mathcal{\theta})}\left(\prod_{a \in V} \frac{1}{\sum_{i=0}^{m-1}|w_i|_{a}^{\sum_{i=0}^{m-1}|w_i|_{a}}}\right) 
		\left( \prod_{a\rightarrow y \in P(d)} \#_{d,a\rightarrow y }^{\#_{d,a\rightarrow y }}\right)$
		\item $\omega = w_0$ , $P^{*} = P(d^*)$, and $p^*(a\rightarrow y) = \frac{\#_{d^*,a\rightarrow y }}{\sum_{i=0}^{m-1}|w_i|_{a}}$ for $a\rightarrow y \in P^{*}$.
	\end{enumerate} 
\end{cor}
\section{Determining the Optimal S0L-System}\label{sec:main-thm2}

In this section, we will focus on answering the second question. The second question involves constructing a S0L-system $G_p$ over $V$ such that $p(\theta)$ is maximal, given alphabet $V$ and a sequence of strings $\theta = (w_0, \ldots, w_m)$. This question can be mathematically translated to an optimization problem for maximizing a posynomial objective function, which is the overall probability of an L-system generating sequence $\theta$ with linear constraints. The linear constraints capture the restriction that the probability of different productions for a fixed character should sum up to one. Given this intuition, we are now ready to present Theorem~\ref{thm:main2}, which offers a formal translation of the second problem to the language of posynomial optimization.

\begin{thm}\label{thm:main2}
	Let sequence $\theta=(w_0, \ldots, w_m)$, and $\mathcal{L}_{\theta}=(V, \omega, P)$ to be the free 0L
	system over $\theta$, then $G^*=(V, \omega^*, P^*,p^*)$ has highest probability of generating $\theta$ if:
	
\begin{enumerate}
		\item\label{main-thm-itm1} $\omega^*=\omega$
		\item\label{main-thm-itm2} $P^* = \{ a \rightarrow y \in P: X_{a\rightarrow y}^* \neq 0\}$ where $\{ X^*_{a\rightarrow y} : a \rightarrow y \in P\}$ is any solution to the following optimization problem:		
	
	\begin{align*}
		&\textbf{Maximize: }  \sum_{d \in  Der_{\mathcal{L}_{\theta}}(\mathcal{\theta})}  \prod_{j=0}^{m-1} \prod_{l=1}^{|w_j|} X_{\sigma_{d}(j,l)}\\
		&\textbf{Subject to, for all $a \in V$: $\left(\sum_{a\rightarrow y \in P} X_{a\rightarrow y}= 1\right)$ and   $X_{a\rightarrow y} \geq 0$ for all $a\rightarrow y \in P$ }\\	
\end{align*}

	\item\label{main-thm-itm3} for all $a \rightarrow y \in P^*$, $p^*(a\rightarrow y) = X^*_{a\rightarrow y}$	
\end{enumerate}
\end{thm}

\begin{proof}
	Any L-system generating $\theta$ needs to have $w_0$ as the axiom; therefore, Item 1 holds. Next we need to show that $G^*$ is indeed an S0L, which means that for each $a \in V$ , $\sum_{a\rightarrow y}p^*(a\rightarrow y)=1$ and $p^*(a\rightarrow y)\geq 0$ for all $a\rightarrow y \in P$. The last claim can be deduced simply from $p^*(a\rightarrow y)= X^*_{a\rightarrow y}, a\rightarrow y \in P^*$ solution to the optimization problem while satisfying the constraints.

	Next, we need to show S0L system \( G^* \) has the highest probability of generating \( \theta \). Consider an arbitrary S0L system \( G_1 = (V_1, \omega, P_1, p_1) \), and notice that we can assume \( V_1 = V \), since any production in \( P_1 \) that is of the form \( a \rightarrow y \) where $a$ or $b$ are not in $V$ only decreases the probability of $G_1$ generating $\theta$, and therefore, $G_1$ can be adjusted so that $V_1 = V$ while either preserving or increasing the probability
	of $G_1$ generating $\theta$. Therefore, we can assume $V_1 = V$. Furthermore, if there are productions $a\rightarrow y \in P \setminus P_1$, we will add them to $G_1$ and define $p_1(a\rightarrow y)=0$ on such productions, and notice that this adjustment does not change the probability
	of $G_1$ generating $\theta$. Take $Y_{a\rightarrow y}=p_1(a\rightarrow y)$ for all $a\rightarrow y \in P$ , and since $G_1$ is an S0L system, we have that $Y_{a\rightarrow y}$ satisfies the constraints of the optimization problem in the statement of this theorem. Therefore, since $X^*$  was an optimal solution, we
	have the following:
	
	\begin{eqnarray}
		\nonumber p^{*}(\theta) &=& \sum_{d \in Der_{G^*}(\mathcal{\theta})}  \prod_{j=0}^{m-1} \prod_{l=1}^{|w_j|} p^*(\sigma_{d}(j,l))\\
		\nonumber  &=& \sum_{d \in Der_{G^*}(\mathcal{\theta})}  \prod_{j=0}^{m-1} \prod_{l=1}^{|w_j|} X^*_{\sigma_{d}(j,l)}\\
		\nonumber  &=& \sum_{d \in Der_{\mathcal{L}_{\theta}}(\mathcal{\theta})}  \prod_{j=0}^{m-1} \prod_{l=1}^{|w_j|} X^*_{\sigma_{d}(j,l)}\\
		\nonumber &\geq&  \sum_{d \in Der_{\mathcal{L}_{\theta}}(\mathcal{\theta})}  \prod_{j=0}^{m-1} \prod_{l=1}^{|w_j|} Y_{\sigma_{d}(j,l)}\\
		\nonumber &=&  \sum_{d \in Der_{G_1}(\mathcal{\theta})}  \prod_{j=0}^{m-1} \prod_{l=1}^{|w_j|} Y_{\sigma_{d}(j,l)}\\
		\nonumber &=& p_1(\theta).
	\end{eqnarray}
	Since $G_1$ was arbitrary, then we have that $G^*$ has highest probability of generating $\theta$ among S0L systems finishing the proof of the Theorem.
\end{proof}
\section{Algorithm}\label{sec:alg}

Given Theorem~\ref{thm:main2}, we are now prepared to present Algorithm~\ref{alg:opt-l} for constructing an optimal stochastic L-system based on a sequence $\theta$ in order to solve Question 2. The algorithm first scans all the characters in the sequence and constructs the axiom by using the first word of $\theta$. Then, as hinted by Theorem~\ref{thm:main2}, the problem is translated into an optimization problem to which we can apply a nonlinear programming solver. Once the optimal solution is found, the probability for each production is substituted in.
\begin{algorithm}[H]
	\caption{Stochastic L-System Inference}
	\label{alg:opt-l}
	\begin{algorithmic}[1]
		\Procedure{InferStochasticLSystem}{$w_0, w_1, \dots, w_m$}
		\State $V \gets \{w_{i}[j] \mid 0 \leq i \leq m , 1 \leq j \leq |w_{i}|\}$
		\State $\omega \gets w_0$
		\State $\theta \gets (w_0, w_1, \dots, w_m)$
		\State construct $\mathcal{L}_{\theta}=(V, \omega, P)$
		\State solve the optimization problem in Theorem ~\ref{thm:main2} , and let $\{ X^*_{a\rightarrow y} : a\rightarrow y \in P \}$ be an optimal solution
		\For{$a\rightarrow y \in P$}
		\If{$X^*_{a\rightarrow y} \neq 0$}
		\State append $a \rightarrow y$ to $P^*$
		\State $p^*(a\rightarrow y) \gets X^*_{a\rightarrow y}$
		\EndIf
		\EndFor 
		\State \textbf{return} $(V, \omega, P^*, \pi^*)$
		\EndProcedure
	\end{algorithmic}
\end{algorithm}
To verify correctness, we recognize that $V$ must include all characters appearing in all strings in the sequence, except possibly the last one. Additionally, the algorithm correctly selects the first string as the axiom. The algorithm then needs to find the S0L system with the highest probability of generating $\theta$. According to Theorem \ref{thm:main2}, we know that the probability of an S0L system generating $\theta$ is:

\[\sum_{d \in  Der_{\mathcal{L}_{\theta}}(\mathcal{\theta})}  \prod_{j=0}^{m-1} \prod_{l=1}^{|w_j|} X_{\sigma_{d}(j,l)}.\]

The algorithm uses this optimal solution to determine the probabilities of different productions occurring. Once we find these probabilities, those that are 0 simply indicate productions that should not be included in our production set. Note that the linear constraints in the nonlinear programming problem ensure that $X_{a\rightarrow y}$ values correspond to production probabilities and must sum to one for each character. After finding After finding $X^*_{a\rightarrow y}$, we only select productions where $X^*_{a\rightarrow y} \neq 0$. This optimal solution maximizes the probability of our S0L system generating $\theta$, thus answering question 2. It's important to note that this algorithm refers to $\mathcal{L}_{\theta}$, whose construction requires polynomial time work with respect to the size of $\theta$.

\section{Conclusion}

In conclusion, this study has significantly enhanced the understanding and practical application of stochastic L-system construction. The theorems presented in this paper not only address the construction of these systems given a sequence but also guarantees the highest probability of generating the specified sequence. The two main questions in this area were targeted which are as following:

\begin{enumerate}
	\item Given a sequence, how can we construct an S0L and a derivation that generate the sequence with the highest probability?
	\item How can we construct an S0L system that, for a given sequence, has the highest probability of generating the sequence?
\end{enumerate}
These questions were answered through the theorems we proved in this study. Furthermore, we proposed an algorithm that utilizes advanced optimization techniques based on nonlinear programming solvers. This algorithm serves as a practical tool for implementing the theoretical insights presented in this study, enabling the stochastically optimal construction of L-systems. However, there is still an open challenge to further enhance the algorithm’s speed and devise faster methods for constructing stochastically optimal L-systems given a sequence.

\bibliographystyle{plain}
\bibliography{test}

\end{document}